\documentclass[letterpaper, 10pt, conference]{ieeeconf}
\IEEEoverridecommandlockouts 
\usepackage{subcaption}
\usepackage[font=small]{caption}
\usepackage{subcaption}
\usepackage{tikz}
\usepackage{xcolor}
\usepackage{graphicx}
\usepackage{amsmath,amssymb}

\usepackage[linesnumbered,ruled,vlined]{algorithm2e}
\SetAlFnt{\small}
\SetAlgoNlRelativeSize{-1}
\SetKwInOut{KwIn}{Input}
\SetKwInOut{KwOut}{Output}
\SetKwFor{For}{for}{do}{end for}



\let\labelindent\relax
 
\usepackage{amsthm}
\usepackage{amsfonts}
\usepackage{amsmath}
\usepackage{amssymb}

\usepackage{float}
\usepackage{optidef}
\usepackage[normalem]{ulem}
\usepackage{lipsum}
\usepackage{mathtools}
\usepackage{cuted}
\usepackage{breqn}
\usepackage{graphicx}
\usepackage{psfrag}
\usepackage[bb=dsserif]{mathalpha}
\usepackage{cite}

\newcommand\redsout{\bgroup\markoverwith{\textcolor{red}{\rule[0.5ex]{2pt}{0.4pt}}}\ULon}

\definecolor{mygreen}{rgb}{0.10,0.50,0.10}

\newtheorem{proposition}{Proposition}
\usepackage{amsthm}
\makeatletter
\let\NAT@parse\undefined
\makeatother
\usepackage{url}
\usepackage[colorlinks,citecolor=blue,linkcolor=red]{hyperref}
\usepackage{cite}
\usepackage{cleveref}
\crefrangeformat{equation}{(#3#1#4)--(#5#2#6)}
\usepackage{import}
\usepackage{enumitem}

\usepackage{booktabs}
\title{\LARGE \bf Diffusion-Based Optimization for Accelerated Convergence of Redundant Dual-Arm Minimum Time Problems}

\author{
    Jushan Chen$^{*}$,
    Jonathan Fried, and
    Santiago Paternain
    \thanks{$^{*}$ Corresponding author. The authors are with the Department of Electrical, Computer, and Systems Engineering, Rensselaer Polytechnic Institute, 110 8th St, Troy, NY 12180, USA. {\tt\small \{chenj72, friedj2, paters\}@rpi.edu}}
}
\begin{document}

%% Shrink space around figures.
% \setlength{\textfloatsep}{0pt}
% \setlength{\textfloatsep}{20pt plus 2pt minus 4pt}
\setlength{\textfloatsep}{5pt plus 2pt minus 4pt}
\setlength{\dbltextfloatsep}{3pt}
% \setlength{\intextsep}{5pt}
% \setlength{\abovecaptionskip}{5pt}
% \setlength{\belowcaptionskip}{3pt}
% \setlength{\parskip}{0pt}
% %% Shrink space around equations
% \setlength{\abovedisplayskip}{3pt}
% \setlength{\belowdisplayskip}{3pt}
% \setlength\abovedisplayshortskip{3pt}
% \setlength\belowdisplayshortskip{3pt}
% %% Shrink space globally around enumerate and itemize
\setlist{nosep} 

\newcommand\green[1]       {{\color[rgb]{0.10,0.50,0.10}#1}}

\maketitle
\thispagestyle{empty}
\pagestyle{empty}

%%%%%%%%%%%%%%%%%%%%%%%%%%%%%%%%%%%%%%%%%%%%%%%%%%%%%%%%%%%%%%%%%%%%%%%%%%%%%%%%%%%%%%%
\begin{abstract}
% In this work, we present a method for minimizing the time required for a redundant dual-arm robot to follow a desired relative Cartesian path at constant path speed by optimizing its joint trajectories, subject to position, velocity, and acceleration limits. The problem is reformulated as a bilevel optimization whose lower level is a convex, closed-form subproblem that maximizes path speed for a fixed trajectory, while the upper level updates the trajectory using a single-chain kinematic formulation and the subgradient of the lower-level value. Numerical results demonstrate the effectiveness of the proposed approach.

We present a framework leveraging a novel variant of the model-based diffusion algorithm to minimize the time required for a redundant dual-arm robot configuration to follow a desired relative Cartesian path. Our prior work proposed a bi-level optimization approach for the dual-arm problem, where we derived the analytical solution to the lower-level convex sub-problem and solved the high-level nonconvex problem using a primal-dual approach. However, the gradient-based nature leads to a large computation overhead, and it prohibits directly imposing an $L_{\infty}$ Cartesian error constraint along the joint trajectory due to the sparsity of the gradient. In this work, we propose a diffusion-based framework that relies on probabilistic sampling to tackle the aforementioned challenges in the nonconvex high-level problem, leading to a 35x reduction in the runtime and 34\% less Cartesian error compared to our prior work.

\end{abstract}

%%%%%%%%%%%%%%%%%%%%%%%%%%%%%%%%%%%%%%%%%%%%%%%%%%%%%%%%%%%%%%%%%%%%%%%%%%%%%%%%%%%%%%%
\section{INTRODUCTION} \label{sec:introduction}
%\textcolor{purple}{\textbf{Just a copy paste. Still working on it.} Should I try to take more or less half a page to talk about robots so Randy can talk about diffusion?}

Redundancy resolution is a fundamental aspect of trajectory planning and control for robotic manipulators. In several industrial tasks, the end-effector's Cartesian path is described with fewer degrees of freedom than the manipulator possesses, either because the robot has more than six joints~\cite{bai2023, elias2024} or the task itself is lower-dimensional~\cite{conkur_1997, leger2016}.  For instance, in applications like spray coating and additive manufacturing~\cite{yundou2018, wen2020, coutinho2023}, rotation around the tool center point is often a free axis. The use of a single robot in geometric complex curves is generally insufficient since it may force the robot to operate in unfavorable joint positions or near singularities~\cite{siciliano1990}. Thus, manufacturers are increasingly adopting dual-arm systems, which increases the redundancy. 

The main advantage of kinematically redundant systems is the capability to optimize multiple metrics, such as velocity and acceleration norms ~\cite{jasour2009,jin2023}, tracking error~\cite{zhang2023}, or fault and singularity avoidance~\cite{li2020,kemeny2003}. On the other hand, the disadvantage of kinematically redundant systems is the increasing complexity of finding an optimal trajectory - a problem that has a finite number of solutions for a non-redundant robot, but has infinite solutions for a redundant system.  %\red{However, maintaining constant path velocity is often a critical requirement to ensure uniform material deposition and process quality~\cite{xiong2020, coutinho2023, Li2018, DING2023683, Tan2017, SERAJ201924}.} \blue{The transition here is not very smooth}\blue{Remove a few references here.}
%\green{A key requirement in additive manufacturing is to guarantee constant path velocity as a proxy for uniformity of material deposition and process quality.} \blue{I tried to provide an alternative, but I think it doesn't fit. The first paragaraph is about redundancy resolution and how it allows for optimizing different metrics. The second paragraph is about the metric we want to optimize. I think the green below fits better.}

A key industrial objective is minimizing the time to complete a given trajectory~\cite{zhang2013, wang2025}, subject to joints position, velocity, and acceleration limitations. Additionally, path speed is often used as a proxy for material deposition in additive manufacturing \cite{xiong2020, coutinho2023, Tan2017, SERAJ201924}, requiring a constant path speed to ensure a uniform deposition~\cite{he2023}. From a motion planning perspective~\cite{reiter2016}, traditional solutions utilize general inverse kinematics methods~\cite{daniel1969, wampler1987} to separate redundancy from trajectory optimization. This typically leads to non-convex problems due to non-linear forward kinematics, forcing reliance on generic solvers~(e.g., \cite{byrd2000trust, waltz2006interior, gill2019practical, Andersson2018}) that cannot exploit the specific structure of path traversal optimization. 

%\blue{Is there anything of the solutions above that does not apply to dual-arms? Even though the problem has been solved for single arms? My rationale is that the justification of dual-arm should come earlier.}

Current research in dual arm systems addresses these metrics through adaptive control~\cite{jiang2022} to optimize contact force, bimanual asymmetry inspired by human motion~\cite{lee2015} to split tasks in coarse motion and fine control categories for each arm, or motion planning using quadratic programming with inequality constraints~\cite{chaki2024} for Cartesian error reduction. Particularly, efforts to minimize execution time in dual-arm tasks include combining Bi-directional RRT with LSTM networks~\cite{YING2021}, MILP-based pick-and-place optimization~\cite{kurosu2017}, and time-optimal control for force manipulation~\cite{DILEVA2024}. However, these methods often focus on trajectory generation or task assignment rather than fully exploiting joint redundancy for precision and speed.

We address dual-arm motion planning by minimizing traversal time under strict accuracy constraints. \cite{bilevel_jonathan} describes a bi-level approach that separates the optimization of the system dynamics (related to path speed) from the kinematics (related to the geometric path and the redundancy resolution). Given a fixed set of joint trajectories, the resulting low-level problem is convex, and the optimal path speed can be efficiently computed in closed form. The remaining high-level problem aims to change the set of joint trajectories in order to improve the maximum path speed. This high-level problem is still non-convex, albeit benefiting from a reduced set of constraints. In previous works \cite{tang2019,sun2021,fried2024,bilevel_jonathan}, this structure is leveraged in gradient-based algorithms that obtain locally optimal solutions.%to obtain gradient-based solutions to the reduced problem, which obtains locally optimal solutions. 

%\blue{We should take a step back and explain this a bit more I think. From the perspective of a reader, what is this bi-level approach? Explain instead that given a joint path, the the path velocity can be computed in closed form. Then, the problem that remains to solve is still non-convex but with less constraints.} \textcolor{purple}{I'm unsure about how much technical detail should go in the introduction, though.} 

% \textcolor{purple}{Diffusion goes here?}
In the years since 2020, the robot learning research community has witnessed a surge in the application of foundation models. In particular, variants of Diffusion Models \cite{Ho2020DenoisingDP,DDIM} have been proposed and validated on various task domains such as manipulation \cite{diffusion_policy, Wu_2025_ICCV}, trajectory planning \cite{safe_diffuser,pan2024modelbaseddiffusion,zhang2025constraineddiffusers,kurtz2024equalityconstraineddiffusion, jiang2025streaming, Chen2025PADTROPD}, and control tasks \cite{kurtz2026generative, zhou2026diffusion}. The backbone of diffusion lies in probabilistic sampling, and it has shown robust performance to mitigate sub-optimality issues in nonconvex optimization. 
% \red{Randy, add a comment explaining why this is promising for our non-convex optimization part. And smooth out the transition to the next paragraph.}

Inspired by diffusion models, in this work, we improve the solutions of the nonconvex high-level problem in the dual-arm motion planning problem by utilizing the low-level subproblem as a cost function for the remaining non-linear high-level problem, which is then solved via a model-based diffusion optimization approach. This solution is still locally optimal, but it obtains an order of magnitude faster and with less Cartesian error than previous methods. Furthermore, this approach simplifies the implementation of non-differentiable constraints, a problem often faced in gradient-based approaches \cite{bilevel_jonathan,sun2021}.

This work is organized as follows. In Section \ref{sec:problem_formulation}, we formulate the redundant dual-arm minimum time constant path speed problem, showcasing the bi-level structure. Its lower-level problem has been shown to have a closed-form solution, yet the remaining high-level problem is still nonlinear. {Section \ref{sec:background} introduces the necessary background on model-base diffusion and Section \ref{sec:methodology} presents the proposed algorithm. Other than concluding remarks (Section \ref{sec:conclusion}), the paper finishes with numerical experiments in Section \ref{sec:results}, where we compare the proposed algorithm to the bi-level approach in our previous work \cite{bilevel_jonathan} in an experiment inspired by a cold spraying application. Comparisons with general off-the-shelf nonlinear solvers (CasADi, ipopt) have been a subject of our previous paper and are not included here for brevity. } %\textcolor{purple}{Section \ref{} exploits efficiency of the low-level solution to do something about the high-level problem?}. \textcolor{purple}{Diffusion benefits when compared to the previous paper goes here?} \blue{I think it should be before the organization. Before organization of the paper we talk about our contribution. } This is supported by the numerical results in Section \ref{sec:results}. In particular, we compare the current approach with the  used in our previous paper \cite{bilevel_jonathan} in an experiment inspired by a cold spraying application, for a given number of initial conditions. \textcolor{purple}{Comparisons with general off-the-shelf nonlinear solvers (CasAdi, ipopt) have been a subject of our previous paper and are not included here for brevity.} This is followed by concluding remarks in Section \ref{sec:conclusion}.

\section{PROBLEM FORMULATION} \label{sec:problem_formulation}
Consider the problem of optimizing the joint motion of two redundant robot manipulators to trace a desired Cartesian path. To be formal, let $s_i$, where $s_0=0$ and $s_N=1$ with $i=0,\ldots,N$ denote the path length and %{\color{red}{$\chi^d: [0,1]\to \mathbb{R}^{m \times N+1}$}} 
$(\chi^d_i)_{\mathcal{F}} \in \mathbb{R}^m$ with $i=0,\ldots N$ the Cartesian path to trace in some frame of reference $\mathcal{F}$. 

Let $n_A$ and $n_B$ be the number of joints of each robot. The robotic system (shown in Figure \ref{fig:transforms}) is redundant with respect to the task, so $n = n_A + n_B > m$. 
Denote by $\mathcal{F}_{bA}$, $\mathcal{F}_{tA}$ the Manipulator A base frame and TCP frame, respectively, and $\mathcal{F}_{bB}$, $\mathcal{F}_{tB}$ those of Manipulator B (see Figure \ref{fig:transforms}). 
% \red{The system and important coordinate frames are shown in Figure \ref{fig:transforms}, where $\mathcal{F}_{bA}$, $\mathcal{F}_{tA}$ are the Manipulator A base frame and TCP frame, respectively, and $\mathcal{F}_{bB}$, $\mathcal{F}_{tB}$ are the Manipulator B base frame and TCP.} 
% \begin{figure}
%    \centering
%    \includegraphics[width=0.9\linewidth]{figures/robot.drawio_6 _2.png}
%    \caption{Setup example consisting of two arms with 3 degrees of freedom.}
%    \label{fig:frameexample}
% \end{figure}

Without loss of generality, it is possible to describe the relative pose between the two end effectors as one single kinematic chain $\mathcal{F}_{tB} \rightarrow \mathcal{F}_{bB} \rightarrow \mathcal{F}_{bA} \rightarrow \mathcal{F}_{tA}$ \cite{roberts2015}. The single-chain forward kinematics $k:\mathbb{R}^{n}\to\mathbb{R}^m$ maps joint space to Cartesian space, i.e.,
\begin{equation}
    \label{eq:fwd_kin_i} 
    (\chi)_{\mathcal{F}_{tB}}= k(q),
\end{equation}
where $q^T= [q_B^T,q_A^T]^T$ is the vector containing the joint pose of both robots. The manipulator's differential kinematics, which maps joint velocity to Cartesian velocity, is given by
% \red{Let $\dot{\chi}_\ell \in \mathbb{R}^m$ be the TCP Cartesian velocity with respect to the base $\ell$, $\dot{q}_\ell \in \mathbb{R}^{n_\ell}$ the joint velocity vector, and $J_\ell(q_\ell) \in \mathbb{R}^{m\times n_{\ell}}$ the manipulator's Jacobian. Then, the manipulator's differential kinematics\red{, which map joint velocity to Cartesian velocity,} are given by}
%
\begin{equation}
    \label{eq:dif_kin_i}
    (\dot{\chi})_{\mathcal{F}_{tB}} = J(q)\dot{q}.
\end{equation}

Representing the pose in $\mathcal{F}_{tB}$ provides benefits twofold. (i) The desired Cartesian path $(\chi_d)_{\mathcal{F}_{tB}}$ is agnostic to the robot's joint pose; and (ii) the relative Jacobian, $J$, defined in \eqref{eq:dif_kin_i} can be calculated from the Jacobian of each manipulator, generally provided by robot manufacturers. For brevity, we forgo writing the reference frame $(\;\;)_{\mathcal{F}_{tB}}$ for the remainder of the text.

The dual-arm robot system is redundant, hence multiple trajectories in joint space result in the same relative Cartesian path. Among these, we are interested in finding one with the minimum traversal time and constant path velocity. Let $\dot{s}$ denote the path velocity; then the path completion time is given by  
%
 %Let $t_i$ with $i=0,\ldots N-1$ denote the time interval needed to go from path point $s_i$ to path point $s_{i+1}$. Then, we can write $t_f$ and its first order approximation as 
\begin{equation}
\label{eq:objective}
t_f = \frac{1}{\dot{s}}.
\end{equation}
%
%
%Note that the objective does not depend on $\dot{s}_N$, since the path is completed at point $N$. Thus, the speed at that point does not contribute to the total time. 

Each joint of the combined manipulator $j=1,\ldots,n$ operates under position, velocity, and acceleration constraints, given by 
\begin{equation}
    \underline{q}_{j} \leq q_{ij} \leq \overline{q}_{j},\;\dot{\underline{q}}_{j} \leq \dot{q}_{ij} \leq \dot{\overline{q}}_{j}\;\mbox{and}\;    \ddot{\underline{q}}_{j} \leq \ddot{q}_{ij} \leq \ddot{\overline{q}}_{j},
    \label{eq:joint_restrictions}
\end{equation}
where $q_{ij}$ is the manipulator's joint $j$-th position at path point $s_i$, $\forall \, j =\{1,\ldots,n\}$ and $\forall \,i =\{0,\ldots,N\}$, and $\dot{\underline{q}}_{j}, \ddot{\underline{q}}_{j} < 0$, $\dot{\overline{q}}_{j}, \ddot{\overline{q}}_{j} > 0$. This assumption implies that revolute joints can rotate clockwise or counterclockwise, or move forward and backward for prismatic joints.  We approximate the joints' trajectory $q_{ij}$ linearly with respect to vectors of parameters $\theta_{j} \in \mathbb{R}^d$ for all $j=1,\ldots{n}$ so that the position, velocity, and acceleration of each joint can be described as
\begin{align}
    q_{ij} = p(s_i)\theta_{j}, \label{eq:joint_reconstruction}\\
    \dot{q}_{ij} = q_{ij}^{\prime}(s_i)\dot{s}_i = p^{\prime}(s_i)\theta_{j} \dot{s}_i. \\
    \ddot{q}_{j} = q_{ij}^{\prime\prime}\dot{s}_i^2 =  p^{\prime\prime}(s_i)\theta_{j} \dot{s}^2.
    \label{eq:paramterization}
\end{align}
where $p(s_i) \in \mathbb{R}^{1\times d}$ is a twice differentiable vector basis in $s$; in this work, a monomial basis for simplicity. We use the fact that $\ddot{s}_i = 0$  since we are interested in trajectories with constant path velocity to simplify the second derivative. It is worth pointing out that linear parameterization assumptions are common in the literature (see e.g., \cite{koubiaspline,zhu2022,embry2018})% and that each joint trajectory could be parameterized by a different basis, however, we use the same one for simplicity. \blue{We can change the last comment now that we define a specific function.} %%
\begin{figure}
    \centering
    \includegraphics[width=9cm]{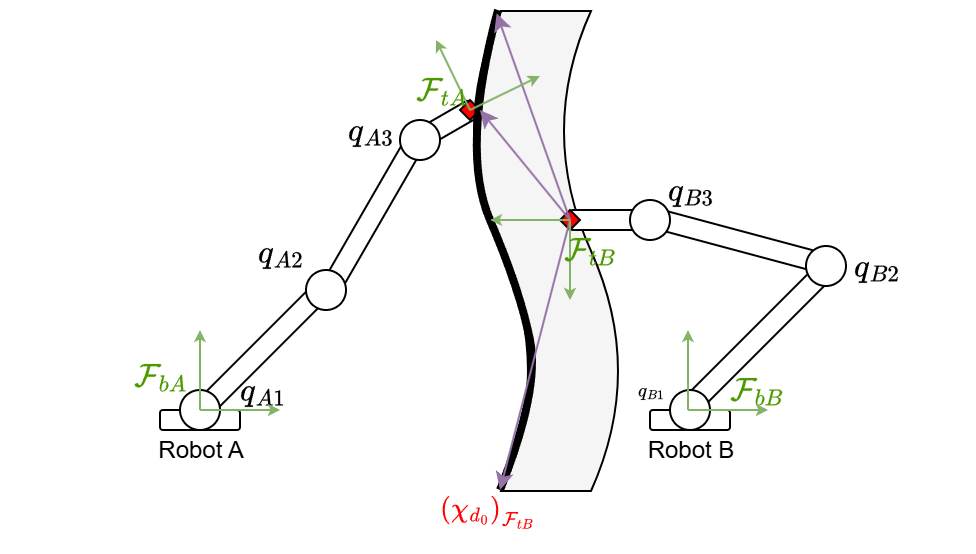}
    \caption{Setup example consisting of two arms with 6 degrees of freedom. Transformations between the frames of the dual arm robotic system.}
    \label{fig:transforms}
\end{figure}
Furthermore, we define the following notation $p_i\!\!=\!\!p(\!s_i\!), p^\prime_i\!\!=\!\!p^\prime\!(\!s_i\!), p^{\prime\prime}_i\! =\!p^{\prime\prime}\!(\!s_i\!)$, which we will use for the remainder of this work. 
It is important to note that since $p^\prime$ and $p^{\prime\prime}$ are not constant, the joint velocities and accelerations are not constant despite $\dot{s}$ being so. Thus, the joint limits must still be considered.

%Furthermore, we write compactly manipulators $\ell$'s joints positions at path point $i$ as $q_{\ell i}\!=\!\left(\!I_{n_\ell}\otimes p_i\!\right)\theta_\ell$, where $\otimes$ denotes the Kronecker product, $I_{n_\ell}$ is a $n_\ell\times n_\ell$ identity matrix and $\theta_\ell = \left[\theta_{\ell 1}^T,\ldots,\theta_{\ell n_\ell}^T\right]^T$. %In what follows, we omit the parenthesis since, given the dimensions of the matrices, this is the only possible order of the operations. 

Finally, define the $L_\infty$ error between the desired and the manipulator's relative Cartesian path as
\begin{equation}
    \label{eq:error}
    E(\theta) = \left\|\chi^d(\theta)- \chi(\theta)\right\|_\infty
    \end{equation} %((k_A(\theta_A))_{\mathcal{F}}-(k_B(\theta_B))_{\mathcal{F}}) 
where we have written the Cartesian path as a function of the joint parameters with a slight abuse of notation. %Note that in many industrial applications (\cite{he2023},\cite{coutinho2023}) this relative path $\chi_d$ is defined by a part attached to the TCP of one of the robots. In this case, describing this path in a general frame $\mathcal{F}$ may be non-trivial and depend on the robot pose.
%, and $\theta_A,\theta_B$ are the vectors of parameters corresponding to robots A and B respectively. 

 With this, we are now able to formalize the minimum time joint optimization problem:

\begin{mini}|s|
  {\stackrel{\dot{s}\geq 0,}{\stackrel{\theta \in \mathbb{R}^{nd}}{{}}}
  }
  { \frac{1}{\dot{s}}}{\label{opt:generaldiscrete}}{{t_f}^\star=}
%  \addConstraint{\sum_{i=0}^N}{ \frac{1}{2}\left\|\chi^d_i-k(I_n\otimes p_i\theta) \right\|_2^2\leq \epsilon}{}
   \addConstraint{}{}{E(\theta)\leq \epsilon}
 \addConstraint{}{}{\underline{q}_{j} \leq p_i\theta_j\leq \overline{q}_{j}}  
  \addConstraint{}{}{\dot{\underline{q}}_{j} \leq p_i^\prime\theta_{j}\dot{s}}\leq \dot{\overline{q}}_{j}
  \addConstraint{}{}{\ddot{\underline{q}}_{j} \leq p_i^{\prime\prime}\theta_{j} {\dot{s}^2} \leq \ddot{\overline{q}}_{j}}
 % \addConstraint{}{}{\dot{s}_{i+1} = \dot{s}_i + \ddot{s}_i t_i}
 % \addConstraint{}{}{s_{i+1} = s_i + \dot{s}_i t_i + \ddot{s}_i \frac{t_i^2}{2}}
 \end{mini}
where $E(\theta)$ is the error on the Cartesian as defined in \eqref{eq:error} and $\epsilon$ a tolerance on it. In the previous expressions, we have omitted that the constraints need to hold for all $i=\left\{0,\ldots, N\right\}$, $j\! =\! \left\{1,\ldots,n\right\}$.

The kinematic and dynamic constraints involving the products of different decision variables and the nonlinear relationship between joints and Cartesian paths (which makes the function $E(\theta)$ nonconvex) make problem \eqref{opt:generaldiscrete} a non-convex optimization problem. However, as established in \cite{bilevel_jonathan}, it is possible to reduce the dimensionality of this problem by first solving its dynamic constraints. %\red{Indeed, for a fixed $\theta$, the optimization problem becomes convex with respect to $\dot{s} \geq 0$.} \blue{We don't really care about convexity here, right?} \textcolor{purple}{Since it has a closed form solution, we don't really.}
Let \begin{dmath}\label{eq_solution_inner_cte}
    V(\theta) \!\!=\!\!\!\!\!\! \max_{\substack{
    j \in\{1,\dots,n\} \\ i\in\{0,\ldots,N\}}} \!\!\!\left\{\!\!\left(\!\!\max\left\{\frac{p'_i\theta_{j}}{\dot{\overline{q}}_{j}},\frac{p^{\prime}_i\theta_{j}}{\dot{\underline{q}}_{ j}}\right\}\!\!\right)^2\!\!\!\!\!,  \max\!\!\left\{\frac{p''_i\theta_{j}}{\ddot{\overline{q}}_{j}},\frac{p^{\prime\prime}_i\theta_{j}}{\ddot{\underline{q}}_{j}}\right\}\!\!\!\right\}.
\end{dmath}
Then, problem \eqref{opt:generaldiscrete} can be rewritten as

\begin{mini}|s|
  {\theta \in \mathbb{R}^{n \times d}}{ V(\theta)}{\label{opt:outer}}{({t_f}^\star)^2=}
  \addConstraint{E(\theta)}{ \leq \epsilon}{}
  \addConstraint{\underline{q}_{j} }{\leq p_i\theta_j \leq \overline{q}_{j}}{}  \end{mini}
where in the previous expressions the constraints need to hold for $i=0,\ldots, N$, $j=1,\ldots,n$. The challenge in solving \eqref{opt:generaldiscrete} is therefore reduced to solving what we term the high-problem~\eqref{opt:outer}. In contrast, the lower-level problem refers to the solution of $\dot{s}_{\max}$ in terms of $\theta$ and the evaluation of $V(\theta)$, see \cite{bilevel_jonathan}. In particular, $V(\theta)$ is the square of the optimal traverse time for a given path defined by $\theta$, as seen in \cite[Theorem 1]{bilevel_jonathan}. Thus, the maximum path velocity (see \eqref{eq:objective}) can be computed as 
\begin{equation}
    \dot{s}_{\max}(\theta) = \frac{1}{\sqrt{V\theta)}}.
\end{equation}

 %\blue{I think this version is more direct and it makes more sense for us.}
%\red{Note in \eqref{opt:generaldiscrete} that, for a fixed $\theta$, the objective is convex in $\dot{s}^2$ and constraints are linear. The key idea in this solution is that to find the maximum path speed a manipulator can perform, for a given and fixed $\theta$, it is necessary to identify the point $i$ and joint $j$ that presents the tightest limit for $\dot{s}$, therefore limiting every point in the path and for every joint to its path speed. This can be done by checking which of the $2n (2N+1)$ joint velocity and acceleration constraints is active. Furthermore, $V(\theta)$ is a convex function and $(\dot{s}^2(\theta))^\star=1/V(\theta)$.}

%\blue{I think it's hard for a reviewer to accept ``note that $V(\theta)$ is the square...'' We need to expand the explanation and connect more with our previous work. Which Theorem of the paper guarantees that (11) holds?} \textcolor{purple}{We actually don't prove that the two problems are equivalent in the dual-arm paper, that's only in the TRO we submitted (or the arxiv) \eqref{eq_solution_inner_cte} is proved in Theorem 1. Should I go back to definiting the lower and higher level problems first, and then giving the solution?}

The high-level problem \eqref{opt:outer}, albeit with fewer constraints than the original problem \eqref{opt:generaldiscrete}, still presents a non-convex feasible set, particularly related to nonlinear forward kinematics when computing the Cartesian error constraint - the objective $V(\theta)$ itself is convex with respect to $\theta$. In \cite{bilevel_jonathan}, a primal-dual method was implemented to solve the leftover problem, where the solutions obtained were local and dependent on the initial joint configurations. In \cite{tang2019}, an off-the-shelf NLP solver (SNOPT) is used to minimize a similar, kinematic high-level subproblem. \cite{sun2021} uses an adaptive backtracking line search routine, albeit their method establishes the kinematics as the low-level problem and the dynamics as the high-level problem. {These gradient-based solutions are limited when solving non-convex, non-differentiable problems. In \cite{bilevel_jonathan}, we approximate the infinity norm of the Cartesian error by a two-norm approximation. Although this choice makes the problem more tractable, we are unable to make guarantees about the correlation between the feasibility of the two different norms.} %\blue{I think we need to make explicit that all these solutions are gradient-based and that they struggle with solving non-convex non differentiable problems. We need to also bring the fact that we are looking at the norm-infinity of the error. This should help us connect better with the MBD framework.}

% \textcolor{purple}{In the next section, we leverage the closed-form solution of the low-level problem (which contains the problem dynamics and can be efficiently solved) as a reward function to iterate the high-level problem (kinematics) using a model-based diffusion method. This solution allow us to obtain results comparable to \cite{bilevel_jonathan}, in a fraction of the time with less Cartesian error. \textbf{Randy, adapt this part depending on the results that you get.}}

In the next section, we leverage the closed-form solution of the low-level problem, which contains the problem dynamics and can be efficiently solved, as a reward function to iteratively solve the high-level problem (kinematics) using a model-based diffusion method, constrained directly with the infinity norm of the Cartesian error. This framework allows us to obtain results comparable to \cite{bilevel_jonathan}, with \textbf{35x reduction in the runtime} and \textbf{34\% less Cartesian error} on average.

\section{Background} \label{sec:background}

Unlike most diffusion-based planners, which learn a denoising model from large demonstration datasets, Model-Based Diffusion (MBD) \cite{pan2024modelbaseddiffusion} performs trajectory optimization directly via probabilistic sampling based on known information about the underlying problem without any training data. For example, we aim to minimize the objective function $\mathcal{R}(\cdot)$. For a given objective function $\mathcal{R}(\cdot)$, we transform it into a probability distribution $P(\cdot)$ following a Boltzmann distribution (energy-like function), given by $P_{\mathcal{R}}(\cdot) \propto e^{\frac{-\mathcal{R}(\cdot)}{\tau}}$, where $\tau$ is a temperature constant. In summary, \textit{model-based diffusion} means that one can evaluate $P_{\mathcal{R}}(\cdot)$ for arbitrary samples in $(\cdot)$. 

While other sampling-based optimization methods already exist, such as Markov Chain Monte Carlo (MCMC) sampling and Cross-Entropy Method (CEM) \cite{Botev2013CrossEntropyOptimization}, the novelty in MBD is that it does not assume the $\mathcal{R}(\cdot)$ or $g(\cdot)$ to be smooth or convex. In MBD, an iterative reverse diffusion procedure is derived that gradually denoises the samples of Gaussian noise into the samples from the target distribution. This is structurally similar to the well-known Denoising Diffusion Probabilistic Models (DDPM) \cite{Ho2020DenoisingDP}, where the score function aggregates samples toward a region of high probability. In contrast, the score function in MBD is not learned from data; instead, it is analytically estimated using Monte Carlo sampling and parallel evaluation of the underlying objective function. To adapt MBD to our problem in~\eqref{opt:generaldiscrete}, the decision variable is $\theta$, which parametrizes the \textit{entire joint trajectory}, for all $i = \{0,\dots,N\}$ and $j = \{1,\dots, n\}$.
% \blue{I'm not sure that this makes sense for us. First we should adapt the language of the paper to ours or clarify the connection better. As a reviewer/reader I would be very confused about this control input. We haven't talked about any control input so far.  Also, our optimization is static to some extent. The parameters $\theta$ parameterize the whole trajectory. We are not solving for each individual $q$ along the path as far as I understood so the method can be interpreted as a diffusion-based analogue of single-shooting trajectory optimization}. 
We denote the reverse diffusion time step as \(t\), where $t = \{\tilde{N}, \tilde{N}-1, \dots,1\}$, with initialization following \(\theta^{\tilde{N}} \sim \mathcal{N}(0,I)^{nd}\). While MBD proposes a trajectory optimization method for dynamical systems, our problem domain is a kinematic trajectory optimization problem. Thus, the reverse update adapted from \cite{pan2024modelbaseddiffusion} is given as
\begin{equation}
    \theta^{t-1}
    =
    \frac{1}{\sqrt{\alpha_t}}
    \left(
        \theta^t + (1-\bar{\alpha}_t)\nabla_{\theta^t}\log P_t(\theta^t)
    \right),
    \label{eqn:MBD_reverse}
\end{equation}
where \(\bar{\alpha}_t = \prod_{k=1}^{t}\alpha_k\) and \(\alpha_t = 1-\beta_t\). In addition, \(\beta_t\) is chosen such that it linearly and uniformly increases from $\beta_0 \ll 1 $ to $\beta_N \ll 1$. Furthermore, \(\{P_t(\cdot)\}_{t=0}^N\) denotes the intermediate probability distributions obtained by gradually corrupting some target distribution \(P_0\) with Gaussian noise. 

At each reverse step \(t\), MBD estimates the score function \(\nabla_{\theta^t}\log P_t(\theta^t)\) by first drawing a batch of \(N_s\) candidate samples of $\theta$ around the current iterate using a Gaussian probability distribution $\mathcal{N}(\cdot)$:
\begin{equation}
    \theta^t_{N_s} \sim \mathcal{N}\!\left(
        \frac{\theta^t}{\sqrt{\bar{\alpha}_t}},
        \left(\frac{1}{\bar{\alpha}_t}-1\right)I
    \right).
    \label{eqn:batch_sampling_mbd}
\end{equation}
Then, each sampled polynomial coefficient vector $\theta^t \in \theta^t_{N_s}$ is then used to reconstruct the joint trajectory following~\eqref{eq:joint_reconstruction} to evaluate the relevant cost functions $\mathcal{R}(\cdot)$ in parallel. Next, samples drawn from~\eqref{eqn:batch_sampling_mbd} are weighted according to the target probability distribution, through the combination of the Boltzmann-like energy distribution \(P_{\mathcal{R}}(\theta)\propto \exp(-\mathcal{R}/\lambda)\). We obtain the following weighted average of the $N_s$ samples at reverse step $t$ as
% \blue{I suggest we remove $g$ since we won't use it.}
\begin{equation}
    \bar{\theta}^{\,t}
    =
    \frac{
        \sum_{\theta \in \theta^t_{N_s}} \theta\, P_{\mathcal{R}}(\theta)\,
    }{
        \sum_{\theta \in \theta^t_{N_s}} P_{\mathcal{R}}(\theta)\
    }.
    \label{eqn:weighted_sample_mean_mbd}
\end{equation}
Using this weighted mean, the score function is finally given as
\begin{equation}
    \nabla_{\theta^t}\log P_t(\theta^t)
    \approx
    -\frac{\theta^t - \sqrt{\bar{\alpha}_t}\,\bar{\theta}^{\,t}}{1-\bar{\alpha}_t}.
    \label{eqn:score_function_mbd}
\end{equation}
Intuitively, the reverse diffusion process gradually aggregates samples around a manifold of \textit{high probability} defined by the underlying target distribution to be minimized, $P_{\mathcal{R}}(\cdot)$. During this process, the current iterate of the denoised polynomial coefficient vector $\theta^t$ is gradually shifted until the target distribution $P_{\mathcal{R}}(\cdot)$ is minimized. At the end of the iterative reverse diffusion process, we obtain the final denoised vector of polynomial coefficients $\theta^{*}$. We refer readers to \cite{pan2024modelbaseddiffusion} for the complete derivation of the score function as well as the proof of convergence guarantee.
% At the end of the iterative reverse diffusion process, the final denoised control sequence \(u^0\) is applied to some generic system dynamics $\mathcal{S}(\cdot)$
% \begin{equation}
%     x_{k+1}=\mathcal{S}(x_k,u_k), \qquad k =0,\dots, T,
% \end{equation}
% which yields a dynamically feasible state trajectory \(x_{0:T}\). 
% Overall, MBD samples from the trajectory-optimality distribution using only the control sequence as the optimization variable, while feasibility is enforced through rollout rather than by diffusing directly in state space.

\section{Model-Based Diffusion for the Outer Subproblem}
\label{sec:methodology}
\subsection{Parameterization of diffusion latent state}
We apply reverse diffusion in the space of polynomial coefficients $\theta \in \mathbb{R}^{nd}$, which are then mapped to joint trajectories using a set of polynomial basis functions. Recall that the joint trajectory is parameterized as $q_{ij} = p(s_i)\theta_{j}$, for $i=0,\dots,N$ and $j = 1, \dots, n$, and each $\theta_j \in \mathbb{R}^d$. Without loss of generality, we assume a set of monomial basis functions with degree $d-1$:
\begin{equation}
    p_k(s) = s_i^{d-k}, \qquad k=0,\dots,d-1.
\end{equation}
 % \blue{Maybe we can define this basis early on and just say that others could be used. }
%Randy: It looks like there is not a better place to introduce the monomial basis since it's a very specific one. I thought about introducing it in Problem Formulation, but that might be too ealry. 
% \blue{Following up the discussion. We can do it above. I don't see a problem. Reviewers have complained in the past that we are not explicit about the basis we use.}

The reverse diffusion process iteratively generates the refined polynomial coefficients $\theta^t$ vector until we obtain $\theta^*$, the solution to \eqref{opt:outer} after $\tilde{N}$ reverse diffusion steps. For numerical stability, we first introduce an auxiliary diffusion variable
\begin{equation}
    y \in [-1,1]^{nd}.
\end{equation}
By denoting $\circ$ as the element-wise product, we then relate $y$ to the polynomial coefficients through the following affine mapping
\begin{equation}
    \theta(y) = \theta^0 + \sigma \circ y,
    \label{eqn:span_mapping}
\end{equation}
where $\theta^0$ denotes some nominal polynomial coefficient vector that we initialize, $\sigma \in \mathbb{R}_+$ controls the exploration range in the space of polynomial coefficients. 
% \blue{Using a word for a variable is not very standard. Let's use a different notation. Greek letter $\sigma$?}
% The nominal coefficients $\theta^0$ can be obtained from a two-step procedure: first, we run inverse kinematics using the desired Cartesian path to obtain a joint trajectory; then, we fit a polynomial of degree $K$ to obtain the corresponding polynomial coefficients. 

Let $\bar{q}_{j} > 0$ denote the magnitude of the joint limit for joint $j$. We choose the element-wise $\sigma$ at joint $j$ using the following expression (see details in Proposition \ref{remark:1})
\begin{equation}
|\theta^0_{j,k}| + \sigma_{j,k} \leq \frac{\bar{q}_{j}}{d},
\label{eq:span_mapping_limit}
\end{equation}
for $k=0,\dots,d-1$, since a polynomial of degree $d-1$ has $d$ distinct terms. This choice of $\sigma_{j,k}$ scales the coefficient explorations in~\eqref{eqn:span_mapping} relative to the given joint limits $\bar{q}_{j}$. Intuitively, by limiting the exploration in the space of polynomial coefficients, we can also limit the joint angle at all path index $i$ along the resulting joint trajectory $q_{ij}$. This choice results in the joint position constraint of problem \eqref{opt:outer} being satisfied. This is the subject of the next proposition.
% \blue{This is part of the proof.}\red{Substituting an joint-wise notation of~\eqref{eqn:span_mapping} into the joint trajectory parameterization $q_{ij} = p(s_i)\theta_j$, we obtain the following element-wise expression for joint angle $j$ at path index $i$:
% \begin{equation}
%     q_{ij}
%     =
%     \sum_{k=0}^{d-1}
%     \Bigl(\theta^0_{j,k} + \sigma_{j,k}\, y_{j,k}\Bigr)^Tp_k(s_i).
% \end{equation}} 

% \blue{Make the remark a proposition. Don't write Equation number, just use \eqref{}, parenthesis and a number already means equation.}
% \blue{We can make relate this proposition to the problem we are solving more direct. }
\begin{proposition}\label{remark:1}
With the condition in~\eqref{eq:span_mapping_limit}, the resulting joint trajectory $q_{ij}$ is always upper bounded by the given joint limit $\bar{q}_{j}$. 
\end{proposition}

\begin{proof} Since each joint is parameterized independently, we prove the result for a single joint $j\in\left\{0,1,\ldots,n\right\}$. We omit the index $j$ in the remainder of the proof to avoid overloading the notation. Recalling the monomial parametrization in \eqref{eq:joint_reconstruction}, write the position of a joint at path point $s_i$ as
\begin{equation}
    q_i = \sum_{k=0}^{d-1} \theta_k^T s_i^{d-k} 
    \label{eqn:joint_traj}
\end{equation}
where the coefficients are parameterized as
\begin{equation}
    \theta_k = \theta^0_{k} + \sigma_k\, y_k,
    \qquad y_k \in [-1,1].
    \label{eqn:coeff_params}
\end{equation}
Substituting~\eqref{eqn:coeff_params} into~\eqref{eqn:joint_traj}, we obtain
\begin{equation}\label{eqn_aux_proposition1}
    q_i
    =
    \sum_{k=0}^{d-1}
    \left(\theta^0_{k} + \sigma_k y_k\right)^Ts_i^{d-k}.
\end{equation}
Note that for all $s_i\in[0,1]$, $\left|s_i^{d-k}\right|\leq 1$ and that $|y_k| \le 1$ during the reverse diffusion updates. Leveraging these facts, and applying the triangle inequality to \eqref{eqn_aux_proposition1} yields
\begin{align}
    |q_i|
    &=
    \left|
    \sum_{k=0}^{d-1}
    \left(\theta^0_{k} + \sigma_k y_k\right)s_i^{d-k}
    \right| \\
    %&\le
    %\sum_{k=0}^{d-1}
    %\left|\theta^0_{k} + \sigma_k y_k\right|
    %|s_i^{d-k}| \\
    &\le
    \sum_{k=0}^{d-1}
    |\theta^0_{k}| + \sigma_k \leq \frac{\bar{q}_{j}}{d}\cdot d = \bar{q}_{j},
    % &\le
    % \sum_{k=0}^{K}
    % \left(|\theta^0_{k}| + \mathrm{span}_k\right),
\end{align}
where the last inequality follows from \eqref{eq:span_mapping_limit}. This completes the proof the result. \end{proof}

% As a result, during the reverse diffusion update, we may safely omit the joint limit constraint $P_g(\cdot)$ in~\eqref{eqn:batch_sampling_mbd}.
% \blue{We need to explain how this results gets us rid of one of the constraints in the problem. I also see how defining the coefficients in this way is harming us because we can never explore all the space. In particular at $s=0$. We need to comment on this. }
% \red{Randy: I think I need to keep the constraint in Equation (14), otherwise I'll need to bring it up again here when discussing why our Proposition removes the constraint afterall.}
It is worth noting that according to Proposition \ref{remark:1}, the exploration size in the space of polynomial coefficients is limited. This choice is intentional due to the fact that the space of polynomial coefficients spans all real numbers. Empirically, we validate that this choice of $\sigma$ leads to superior performance in the following sense: one may choose a larger exploration parameter $\sigma$ by solving an optimization problem to find the exact largest exploration bound in the space of $\theta$ without violating joint limit constraints, but might need to sacrifice performance in the runtime. Qualitatively, a larger exploration parameter $\sigma$ means that it takes longer for diffusion to converge to a high-quality solution.
% \red{leads to superior performance (see Sec \ref{sec:results}).} \blue{as compared to other values? Are we presenting those results?}
% \textcolor{purple}{actually, we only present the results for this choice of $\sigma$. What I mean to say here is that our results are good, even though exploration is limited.}

\subsection{Reverse Diffusion Loop for the High-Level Subproblem}
We adapt the model-based diffusion algorithm \cite{pan2024modelbaseddiffusion} to the high-level problem defined in~\eqref{opt:outer}. We summarize our model-based diffusion algorithm for the high-level problem as shown in Algorithm \ref{alg:proj-diff}. 
% \blue{Explain the connection between the Background section and the algorithm and then explain the differences.} 
While the original MBD algorithm \cite{pan2024modelbaseddiffusion} applies to trajectory optimization of dynamical systems that evolve over time, our diffusion-based algorithm applies to static kinematic trajectory optimization. For our particular problem, instead of applying a static cost function as proposed in \cite{pan2024modelbaseddiffusion} or the quadratic augmented Lagrangian cost function proposed in \cite{kurtz2024equalityconstraineddiffusion}, we propose a novel adaptive cost function for the high-level problem at each reverse diffusion update step $t$ (Line 12, Algorithm \ref{alg:proj-diff}).
% \blue{There was that Lagrangian diffusion method. This is also similar, right? We should make that connection.} Yes, I am including that one as well but it's not using an adaptive reward. It is a fixed augmented Lagrangian function with a quadratic term.
First, we map $y$ to $\theta$ according to~\eqref{eqn:span_mapping}. Then, at the reverse update step $t$, the cost function is given as
% \blue{I added a superscript t since it depends on lambda. Now, the functions V and E, don't change with time. They only change with theta, so making the explicit dependence with the time seems incorrect to me. I suggest we remove all time dependencies. We can also make R depend on Lambda instead of t. That would be more appropriate. }
\begin{equation}
    \mathcal{R}\left(\theta\left(y^{t}\right), \tilde{\lambda}^t\right) = V\left(\theta\left(y^{t}\right)\right) + \tilde{\lambda}^{t}\left( E\left(\theta\left(y^{t}\right)\right)-\epsilon \right),
    \label{eqn:adaptive_reward}
\end{equation}
where $\tilde{\lambda}^t$ is an adaptive weight, and $V(\cdot)$ is obtained by evaluating~\eqref{eq_solution_inner_cte}, while $E(\cdot)$ is obtained by evaluating~\eqref{eq:error}.
% \blue{This is another good moment to explain the term model-based in the context of this work.}
As discussed in Section \ref{sec:background}, our diffusion algorithm is model-based since we first draw samples around the current iterate $\theta^t$ and then weigh the samples by evaluating the \emph{known objective distribution} $P_{\mathcal{R}}(\cdot)$ over these samples.
The adaptive weight $\tilde{\lambda}$ at the next reverse diffusion step is updated as
\begin{equation}
    \tilde{\lambda}^{t-1} = \max \left(0,  \lambda^t -\gamma \left(E\left(\theta\left(y^{t-1}\right)\right)-\epsilon \right)\right),
    \label{eqn:lambda_adaptive}
\end{equation}
% \blue{I think we are missing $\lambda^t$.}
where $\gamma\in(0,1)$ is a fixed step size across all $t$. This resembles the primal-dual optimization method \cite{boyd2004convex}, but in our approach, there is no gradient information. We find that this choice of adaptive weighting provides a stable balance between the penalty on the Cartesian error $E(\cdot)$ and the final time $V(\cdot)$. 
% In addition, since the resulting joint trajectory never exceeds the limit as shown in Proposition \ref{remark:1}, the weighted sample mean in~\eqref{eqn:weighted_sample_mean_mbd} becomes a softmax distribution of $\mathcal{R}(\cdot)$ only, as we can safely omit $P_g(\cdot)$. 
% If there are additional safety constraints that need to be imposed, then we may retain $P_g(\cdot)$ and modify the computation of the weighted mean accordingly.
% \blue{If we remove $g$ from the background description or we comment about it as an additional possibility we can remove this comment.}

\begin{algorithm}[htbp]
\caption{Model-Based Diffusion for Dual-Arm Minimum-Time Optimization}
\label{alg:proj-diff}

\DontPrintSemicolon
\SetAlgoNlRelativeSize{-1}
\SetNlSkip{0.4em}
\IncMargin{0.8em}
\SetInd{0.3em}{0.8em}
\raggedright

\KwIn{
Noise schedule parameters $\{\alpha_t,\bar{\alpha}_t\}_{t=1}^{\tilde N}$,\;
number of samples per reverse step $N_s$,\;
temperature $0s< \tau<1$,\;
desired relative Cartesian path $(\chi^d_i)_{\mathcal{F}}$,\;
high-level problem cost $J(y)$,
}
\KwOut{Optimized diffusion state $y^\star$}

Initialize nominal $\theta_0$ via inverse kinematics and polynomial fit\;
Set $t \gets \tilde N$\;
Initialize diffusion state $y^{t} \sim \mathcal{N}(0,I)$\;

\While{$t \ge 1$}{
\parbox[t]{0.86\linewidth}{%
Sample a batch of size $N_s$ in parallel:\\[1.0ex]
\hspace*{1.2em}$y_{N_s}^{t} \sim
\mathcal{N}\!\left(
\dfrac{y^{t}}{\sqrt{\bar{\alpha}_{t-1}}},
\left(\dfrac{1}{\bar{\alpha}_{t-1}}-1\right)I
\right)$%
}\;

\parbox[t]{0.86\linewidth}{%
Clip values for numerical stability:\\[1.0ex]
\hspace*{1.2em}$y_{N_s}^{t} \gets \mathrm{clip}(y_{N_s}^{t},-1,1)$%
}\;

\parbox[t]{0.86\linewidth}{%
Evaluate the objective for all samples in parallel~\eqref{eqn:adaptive_reward}:\\[1.0ex]
\hspace*{1.2em}$c_{N_s} \gets -\mathcal{R}\!\bigl(\theta(y_{N_s}^{t}),\tilde{\lambda}^t\bigr)$%
}\;

\parbox[t]{0.86\linewidth}{%
Normalize and compute softmax weights:\\[1.0ex]
\hspace*{1.2em}$\hat c_{N_s} \gets
\dfrac{c_{N_s}-\mathrm{mean}(c_{N_s})}
{\mathrm{std}(c_{N_s})\,\tau},
\qquad
w_{N_s} \gets
\dfrac{e^{\hat c_{N_s}}}
{\sum_{\ell=1}^{N_s} e^{\hat c_\ell}}$%
}\;

\parbox[t]{0.86\linewidth}{%
Compute weighted mean:\\[1.0ex]
\hspace*{1.2em}$\bar y^{t} \gets \sum_{\ell=1}^{N_s} w_\ell\, y_\ell^{t}$%
}\;

\parbox[t]{0.86\linewidth}{%
Estimate the score function:\\[1.0ex]
\hspace*{1.2em}$s^{t} \gets
\dfrac{-y^{t}+\sqrt{\bar{\alpha}_t}\,\bar y^{t}}
{1-\bar{\alpha}_t}$%
}\;

\parbox[t]{0.86\linewidth}{%
Reverse update step:\\[1.0ex]
\hspace*{1.2em}$y^{t-1} \gets \dfrac{1}{\sqrt{\alpha_t}}
\left(y^{t}+\left(1-\bar{\alpha}_t\right)s^{t}\right)$%
}\;

Apply adaptive penalty weight $\tilde{\lambda}^{t-1}$~\eqref{eqn:lambda_adaptive}
% $t \gets t-1$\;
% }
}
\DecMargin{0.8em}
\end{algorithm}

% \vspace{-1em}
\section{Results}
\label{sec:results}
In this section, we present numerical experiments to validate the superior optimization performance of our Algorithm \ref{alg:proj-diff} in Section \ref{sec:methodology}. First, we apply Algorithm \ref{alg:proj-diff} as a local optimization solver to compare against the results reported by the bi-level optimization approach proposed in \cite{bilevel_jonathan}, using the same numerical setup. Comparisons with general off-the-shelf nonlinear solvers (CasAdi, ipopt) have been a subject of our previous paper, where the superior performance of the bi-level solution \cite{bilevel_jonathan} has been demonstrated. As such, these are omitted in this section for brevity.

The simulation models a cold spraying task using two planar 3R manipulators, where one of the robots holds the spray, while the other holds the workpiece. The curve $\chi^d$ to be followed represents the leading edge of a 2D fan blade mock-up, discretized into $N+1 = 500$ uniformly spaced points. A Cartesian error tolerance of $\pm5$ mm is set per industry standards.

The length of each manipulator's links are $a_1 = 2m$, $a_2 = 1.5m$, $a_3 = 1m$ and their joint limits are given by $\dot{\overline{q}} = -\dot{\underline{q}} =[1.75\;1.57\;1]\;\;rad/s$ and $\ddot{\overline{q}} = -\ddot{\underline{q}} = [35\;31.4\;20]\;\;rad/s^2$.
These values are based on the FANUC M-1000iA industrial robot to handle the reach and size of the fanblade\footnote{The typical industrial robot of this size would usually have at least 6 degrees of freedom. We have reduced it to a 3R manipulator for the purposes of this numerical demonstration.}. The simulation setup can be seen in Figure \ref{fig:transforms}.

The forward kinematics mapping for each 3R planar robot manipulator is given as follows

    \begin{equation}\label{eqn_forward_kinematics_experiment}
        \left[\begin{array}{c}
             x  \\
             y  \\
             \phi
        \end{array}\right] = \left[\begin{array}{c}
             a_1c_1 + a_2c_{12}+a_3c_{123}  \\
             a_1s_{1} + a_2s_{12}+a_3s_{123}  \\
             q_{1_\ell} + q_{2_\ell} + q_{3_\ell}
        \end{array}\right],
    \end{equation}
    where $c_{ijk} = \cos(q_{i_\ell} + q_{j_\ell} + q_{k_\ell})$, $s_{ijk} = \sin(q_{i_\ell} + q_{j_\ell} + q_{k_\ell})$ and $\ell \in \{A,B\}$ indicates the individual robot arm.

We have a total of $n = 6$ joints and our desired Cartesian trajectory is given by $(x,y) \in \mathbb{R}^2$. We calculate initial conditions by computing analytical inverse kinematics, after solving the redundancy with a heuristic. In this case, we consider a baseline single-arm scenario. The manipulator holding the blade initially holds the piece at the center of the path $s$ with $\phi = \pi$, constant, which fixes its three joints. This leaves one degree of freedom, determined by $\phi$ of the spray-holding arm. For those, we sample values uniformly at random in $[0, 2\pi]$ and held constant along the trajectory.

% \textcolor{purple}{in python?}.

% \subsection{Comparison against bi-level optimization solver}
% \label{sec:resultsI}

We initialize the polynomial coefficients by first solving the analytic inverse kinematics from 8 randomly chosen orientation angles $\phi$, and then parameterizing the resulting joint trajectory $q_{ij}$ with $\theta^0_{j}$ using polynomial fitting, similar to \cite{bilevel_jonathan}. We adopt a 9th-order polynomial, consistent with \cite{bilevel_jonathan}. While the bi-level optimization approach solves the high-level problem (\ref{opt:outer}) using a gradient-based primal-dual approach, our solver applies model-based diffusion. We run $\tilde{N}$ = 200 steps. We show the numerical results from our diffusion solver in Table \ref{tab:diffusion_solver}, and show the numerical results from the bi-level optimization solver \cite{bilevel_jonathan} in Table \ref{tab:bi_level_solver}. We run all our simulations in Python on a PC with an Intel 8th Gen CPU and 16GB of RAM, without any GPU accelerations.

% \textcolor{red}{Results have been updated;tuned the parameters of the adaptive weight on the cartesian error, add 3 more trials where the initial inverse kinematics is reasible}
\begin{table}[t]
\caption{Results from our diffusion solver for different $\phi$}
\label{tab:diffusion_solver}
\centering
\footnotesize
\setlength{\tabcolsep}{5pt}
\renewcommand{\arraystretch}{1.15}
\begin{tabular}{c c c c}
\hline
$\phi$ (rad) & Optimized $t_f$ (s) & Max Cart. Err. (mm) & Runtime (s) \\
\hline
5.12 & 0.73 & 2.83 & 13.77 \\
0.00 & 1.17 & 4.39 & 10.47 \\
1.04 & 0.79 & 2.94 & 11.62 \\
0.56 & 1.01 & 2.19 & 11.08 \\
4.76 & 0.90 & 3.87 & 11.66 \\
1.22 & 0.70 & 3.54 & 12.01 \\
4.98 & 0.78 & 2.94 & 11.01 \\
6.01 & 1.06 & 4.06 & 10.83 \\
\hline
Avg & \textbf{0.89} & \textbf{3.35} & \textbf{11.56} \\
Std & \textbf{0.16} & \textbf{0.69} & \textbf{0.96} \\
\hline
\end{tabular}
\end{table}

\begin{table}[t]
\caption{Results from the bi-level optimization solver for different orientation angle $\phi$}
\label{tab:bi_level_solver}
\centering
\footnotesize
\setlength{\tabcolsep}{5pt}
\renewcommand{\arraystretch}{1.15}
\begin{tabular}{c c c c}
\hline
$\phi$ (rad) & Optimized $t_f$ (s) & Max Cart. Err. (mm) & Runtime (s) \\
\hline
5.12 & 0.77 & 11.01 & 1232.91 \\
0.00 & 0.99 & 5.36 & 625.28 \\
1.04 & 0.80 & 2.52 & 204.39 \\
0.56 & 0.95 & 4.42 & 198.74 \\
4.76 & 0.91 & 2.04 & 217.72 \\
1.22 & 0.73 & 2.42 & 196.65 \\
4.98 & 0.82 & 4.71 & 474.38 \\
6.01 & 0.72 & 8.24 & 156.98 \\
\hline
Avg & \textbf{0.84} & \textbf{5.09} & \textbf{413.38} \\
Std & \textbf{0.09} & \textbf{2.93} & \textbf{346.59} \\
\hline
\end{tabular}
\end{table}
The Maximum Cartesian Error is the point-wise maximum Cartesian error from the optimized trajectory. In practice, due to the gradient-free nature of our diffusion solver we can directly optimize for the Maximum Cartesian Error~\eqref{eqn:lambda_adaptive}, whereas the bi-level optimization solver \cite{bilevel_jonathan} or a well-known Interior Point solver \cite{casadi} struggles with the Maximum Cartesian Error metric due to its gradient-based nature -- evaluating the gradient of infinity norm associated with the Maximum Cartesian Error yields a very sparse vector. Thus, in the bi-level optimization solver, a 2-norm weighted error is used as a surrogate. 
% \blue{There is no 2 norm error in 32. I think we should explain the opposite, that Jonathan's approach (or any gradient based method for all that matter will have an issue with the infinite norm and that for that reason they use weighted penalties which can result in hard to control errors. 
% We have done the test of solving a nondifferentiable minimization problem with IPOPT and it was bad too.} 

Based on these results, we observe that our diffusion solver is able to achieve approximately \textbf{34\% lower} Maximum Cartesian Error on average, with over \textbf{35x reduction} in runtime. In addition, our diffusion solver leads to a small standard deviation in the runtime due to its gradient-free nature ($\sim$ 8\% of the average), whereas the bi-level optimization solver leads to a standard deviation of over 80\% of the average runtime. A tradeoff we observe is that the optimized $t_f$ from our diffusion solver (Table \ref{tab:diffusion_solver}) is, on average, approximately 5\% larger than that of the bi-level optimization solver (Table \ref{tab:bi_level_solver}). However, in several trials, the optimized $t_f$ from our diffusion solver outperforms the bi-level optimization solver.

\section{Conclusion}\label{sec:conclusion}
% \blue{Don't forget to write the conclusions.}
In summary, to tackle the remaining challenges in our prior work \cite{bilevel_jonathan}, we have proposed a novel diffusion-based framework that does not depend on gradient information. As a result, our method achieves a 35x reduction in the runtime and 34\% less Cartesian error. 

In a future study, we would like to further investigate how the size of exploration in the space of polynomial coefficients affects the final outcome of our diffusion algorithm when applied to various robotic arm motion planning problems.
\bibliographystyle{ieeetr}
\bibliography{references}

\end{document}